\documentclass{llncs}
\usepackage{verbatim}
\usepackage{amsmath}
\usepackage{amssymb}
\usepackage{algorithm}
\usepackage{graphicx}
\usepackage[noend]{algpseudocode}
\usepackage{drum}
\usepackage{times}
\usepackage[rightcaption]{sidecap}
\usepackage{url}
\usepackage{adjustbox}
\usepackage[font={small,it}]{caption}

\usepackage[numbers]{natbib}
\usepackage{multicol}

\author{Evan M. Drumwright}
\institute{Toyota Research Institute}

\graphicspath{{images/}}

\begin{document}

% paper title
\title{True Rigidity: Interpenetration-free Multi-Body Simulation with Polytopic Contact}

\maketitle
\begin{abstract}
An effective paradigm for simulating the dynamics of robots that locomote and manipulate is multi-rigid body simulation with rigid contact. This paradigm provides reasonable tradeoffs between accuracy, running time, and simplicity of parameter selection and identification. The Stewart-Trinkle/Anitescu-Potra ``time stepping'' approach is the basis
of many existing implementations. It successfully treats inconsistent (Painlev\'{e}-type) contact configurations, efficiently handles many contact events occurring in short time intervals, and provably converges to the solution of the continuous time differential algebraic equations (DAEs) as the integration step size tends to zero. However, there is currently no means to determine when the solution has largely converged, i.e., when smaller integration steps would result in only small increases in accuracy.  

The present work describes an approach that computes the event times (when the set of active equations in a DAE changes) of all contact/impact events for a multi-body simulation, toward using integration techniques with error control to compute a solution with desired accuracy. 

We also describe a first-order, variable integration approach that ensures that rigid bodies with convex polytopic geometries never interpenetrate. This approach permits taking large steps when possible and takes small steps when contact is complex. 
\end{abstract}

\begin{comment}
- focus in this paper is on rigid (or nearly rigid) contacts, b/c compliance
  parameters often have no basis in reality (they must be tuned to get
  numerical and/or dynamical stability or expensive integrators must be used), 
  because simulations with springs and dampers are rarely faster (and at best
  a factor of two faster), and b/c numerous models exist and the advantage of
  a particular model over any other is rarely clear.  

simulations undergoing intermittent contacts have been realized in one of two
ways: (1) contact points are computed from intersecting features of two geometries, contact forces are computed (and additional force is applied to push the bodies apart) or (2) all pairwise features that may come into consideration are incorporated into the complementarity problem. The first of these is subject to
artifacts as interpenetration becomes significant; the second is subject to
missing contact events, meaning that convergence cannot be identified, and
the scalability of the problem is nasty: $O(n^2^3)$ at minimum.

We describe an alternative algorithm with the following advantages: all events
are captured, interpenetration does not occur. The disadvantages of this 
algorithm are that bodies undergoing sustained contact or scenarios involving
many contact events over a short time interval can exhibit slow running 
times.  

guarantee of no interpenetration due to tunneling either

demonstrate correctness on:
- rimless wheel
\end{comment}

\section{Introduction}
Verifying correctness for contact-free multi-body simulations is straightforward. As examples, one can verify that energy remains constant on a conservative system, can compare state against the closed form solution for a classical system, e.g., a pendulum, or can compare state against trusted numerical solutions (produced using simulation software that has been carefully assessed for correctness).  Verifying correctness for multi-body simulations with contact is much harder. There are few benchmark problems that possess known solutions for multi-body dynamics simulations with rigid contact, and we are aware of no such benchmarks for bodies with polytopic geometries.

Meanwhile, researchers in the robotics, multi-body dynamics, nonsmooth mechanics, and computer graphics communities continually propose approaches for increasing simulation speed. When researchers have assessed correctness of an approach, they have done so by searching for qualitative behavior, like rattleback toy spin reversal~\cite{Mirtich:1996vt} and pattern generation in vibrated bins of granules~\cite{Smith:2012}. We aim to instead find accurate numerical solutions, thereby motivating the present approach. Approaches for simulating multi-rigid body dynamics with contact would then have to produce \emph{the} result as quickly as possible, rather than \emph{a}, perhaps plausible, result (the status quo). 

While we envision the present work as leading directly to such an oracle that can produce the correct result, the current approach is usable for typical multi-body dynamics simulation applications as well: the approach integrates the equations of motion directly to changes in the contact manifold. Since all such changes are detected and interpenetration is prevented, gross artifacts---like objects that are ejected from robots grasps---that frequently appear in robotic simulations should be eliminated. We focus on polytopic geometries, which are not only a popular representation but also allow using the linearized halfspace constraints supported by linear complementarity problem (LCP) formulations without approximation.

%Our motivating application is the pick-and-place task, where large integration steps can be taken during motion control but small steps are required during grasping. However, we devised this approach to address  the gross artifacts, like bodies that are ejected from grasps or are overly constrained by ``invisible'' contacts, that frequently appear in multi-body simulations. 

\section{Related work}

This paper builds upon work from  research in nonsmooth mechanics, computer science, and mathematics.

\subsection{DAEs and DVIs for modeling multi-body dynamics with contact}
\label{section:related:DVI}

Multibody dynamics with rigid contact and Coulomb friction---which captures important stick-slip transitions---can be modeled as a differential algebraic equation (DAE):
\begin{align}
\ddot{\vect{q}} & = \vect{f}(\vect{q}, \dot{\vect{q}}, \vect{u}) \\
%\vect{0} & \le \vect{\phi}(\vect{q}) \label{eqn:unilateral} \\
\vect{0} & = \vect{\phi}(\vect{q})  \label{eqn:DAE2}
\end{align}
%where $\vect{\phi}(.)$ are known as unilateral constraints (and correspond to contact and joint limit constraints) because they constrain movement in only one direction and $\vect{\phi}(.)$ are known as bilateral constraints (and clearly constrain movement along both directions of each constraint). 
where $\vect{\phi}(.)$ is a set of \emph{active} algebraic constraints, out of $m$ total constraints. Some constraints are always active, like bilateral joint constraints. Other constraints are only active if certain conditions are met; e.g., a contact constraint between two polyhedra would only be active when the bodies are in contact at that point and they would otherwise (i.e., without the constraint in place) interpenetrate at that point:
\begin{align}
\dot{\phi}_i(\vect{q}, \dot{\vect{q}}) = 0 \textrm{ if } \phi_i(\vect{q}) = 0 \textrm{ and } \lambda_i > 0, \textrm{ for } i=1,\ldots,m \label{eqn:DAE1}
\end{align} 
where $\lambda_i$ acts as a Lagrange Multiplier (i.e., it is zero if the constraint is active and non-negative otherwise). A standard \emph{index reduction} technique seeks to remove all algebraic variables from a DAE, yielding an ODE system explicit for all unknowns.

We aim to apply error estimation for general ODEs to DAEs for multi-rigid body dynamics with contact.  This requires us to identify \emph{all} such points in time where any constraint becomes active or inactive. Root finding, \software{MATLAB}'s technique for handling such problems, is prone to missing such times as rigid bodies can make and break contact rapidly; if two bodies are not in contact at the endpoints of $[a, b]$ but are in contact at $\frac{1}{2}(a+b)$, root finding will fail to find such \emph{events}. Between any two consecutive events, ODE-based techniques can be applied for estimating the error.

We note that piecewise DAEs with unilateral constraints are generally modeled as differential variational inequalities (DVI)~\cite{Pang:2008}, which---for rigid body contact problems, at least---we find more challenging to pose, though existing solution approaches can work reasonably well. These approaches will be discussed in Section~\ref{section:related-work:time-stepping}.

\begin{figure}[htpb]
\centering
\includegraphics[width=.85\linewidth]{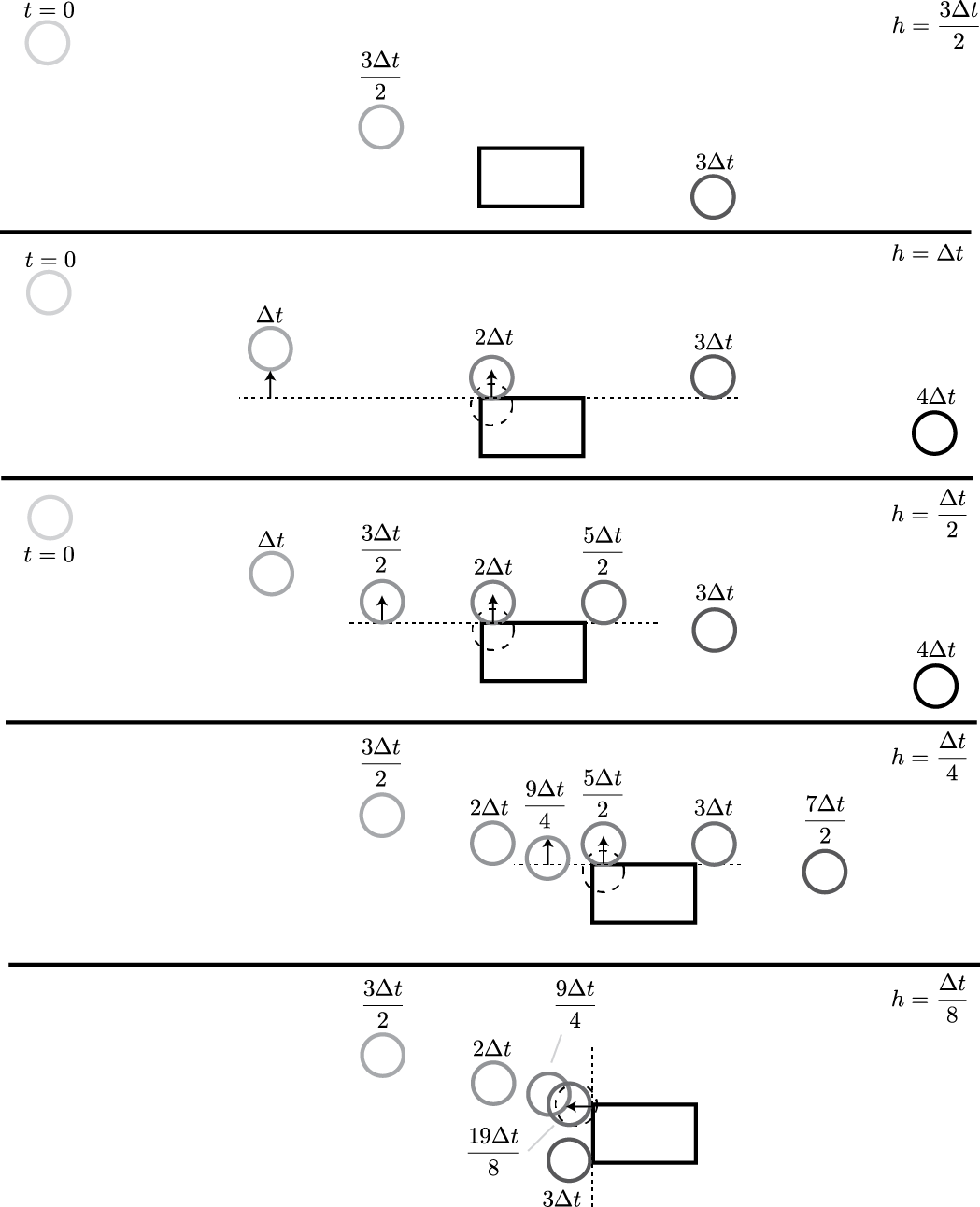}
\caption{Depiction of Stewart-Trinkle for simulating a ball moving ballistically toward a stationary box, with different fixed step sizes (top to bottom).  Only some ball positions are depicted. Contact points lie at the source of arrows; arrows are drawn each time a contact constrains  motion; feasible half spaces are drawn as a line segment with dashed stroke. Contact constraints are determined as described in~\cite{Stewart:1996,Stewart:2000}: when bodies are found to be intersecting at $t+h$ (circle drawn in dashed stroke), the simulation is regressed to $t$ and a contact constraint is added. The motion does not qualitatively match the adaptive result (see Figure~\ref{fig:adaptive-approach}) until the step size is $\Delta t/8$.\vspace{-.25in}
}
\label{fig:ST-approach}
\end{figure}

\subsection{Event-driven approaches to simulation}
Early works in simulating rigid bodies undergoing rigid contact were conducted
by L\"ostedt~\cite{Lostedt:1982,Lostedt:1984} and Baraff~\cite{Baraff:1989sm,Baraff:1991,Baraff:1994}, and used root finding approaches to locate events. These works described some challenges with theory---inconsistent configurations, exemplified by the Painlev\'{e} Paradox~\cite{Painleve:1895} that was well known to the theoretical mechanics community---and computational complexity: for example, Baraff showed that the problem of classifying a contact configuration as inconsistent is NP-hard~\cite{Baraff:1994}. These works do not consider the challenges of locating all possible events or the failure to do so.

\subsection{Time-stepping approaches to simulation with Coulomb friction}
\label{section:related-work:time-stepping}
Stewart and Trinkle~\cite{Stewart:1996} and Anitescu and Potra~\cite{Anitescu:1997} described an approach (based on Moreau's ``time stepping''~\cite{Moreau:1985} discretization of the rigid body dynamics and differential inclusion theory) that provably mitigated the
problem of inconsistent configurations. These authors proved that the complementarity problem~\cite{Cottle:1992} based approaches used in these works do non-positive work and always possess a solution if the contact constraints are linearly independent. Shortly after introducing this method, Stewart proved convergence of the approach to the solution to the continuous time dynamics~\cite{Stewart:1998} as the integration step tends to zero. Stewart-Trinkle~\cite{Stewart:2000} and Anitescu-Potra~\cite{Anitescu:2004} later described approaches that, respectively, prevent and minimize interpenetration; Figures~\ref{fig:ST-approach}~and~\ref{fig:AP-approach} illustrate the similarities and differences. Proofs of convergence and constraint stabilization rely upon assumptions that the complementarity problem has unique solutions~\cite{Stewart:2000} and the \emph{signed distance} function~\cite{Anitescu:2004} is differentiable.

\begin{comment}
\begin{figure}[htpb]
\includegraphics[width=.49\linewidth]{S-T}
\caption{The Stewart-Trinkle approach, if implemented as described in~\cite{Stewart:2000}, is unable to find a solution to this contact configuration. This figure is taken from ...; ``Vertex $v_{a1}$ is penetrating the line containing edge $e_{b3}$ (i.e., $\Psi_1 < 0$). Similarly, $v_{a2}$ is penetrating the line containing edge $e_{b1}$ (i.e., $\Psi_2 < 0$). Since the bodies are rigid, it is impossible to eliminate both penetrations at the end of the next time step.'' This example illustrates linearly dependent linearized contact constraints and will prevent the associated complementarity problem from being solved. \changeme{How does our work address this?}}
\end{figure}
\end{comment}

A significant issue with time stepping approaches comes from determining constraints
that hold over the length of time intervals: inconsistent constraints can cause the complementarity problem to fail to have a solution, and movement between bodies may be unphysically constrained.

\subsection{Conservative advancement}
Mirtich pioneered \emph{conservative advancement} (CA), which uses the distance between bodies and bodies' velocities and accelerations, to integrate rigid bodies and multi-rigid body systems without missing contact events (a guarantee root finding approaches cannot make~\cite{Mirtich:1996vt}).  Assuming the typical first order integration process for multibody dynamics with contact, CA (depicted in Figure~\ref{fig:CA}) can use the formula below, where variables $\dot{\vect{x}}$ represent linear velocities, $\vect{\omega}$ represent angular velocities, $r$ is the minimum bounding sphere for the body, and $\hat{\vect{d}} \equiv \vect{d}/||\vect{d}||$:
\begin{align}
\Delta t^{\textrm{max}} \equiv 
\frac{||\vect{d}||}{|\tr{\hat{\vect{d}}}(\dot{\vect{x}}_A - \dot{\vect{x}}_B)| + ||\vect{\omega}_A \times \hat{\vect{d}}||\, r_A + ||\vect{\omega}_B \times \hat{\vect{d}}||\, r_B} \label{eqn:CA}
\end{align}
\begin{figure}[htpb]
\centering
\includegraphics[width=.8\linewidth]{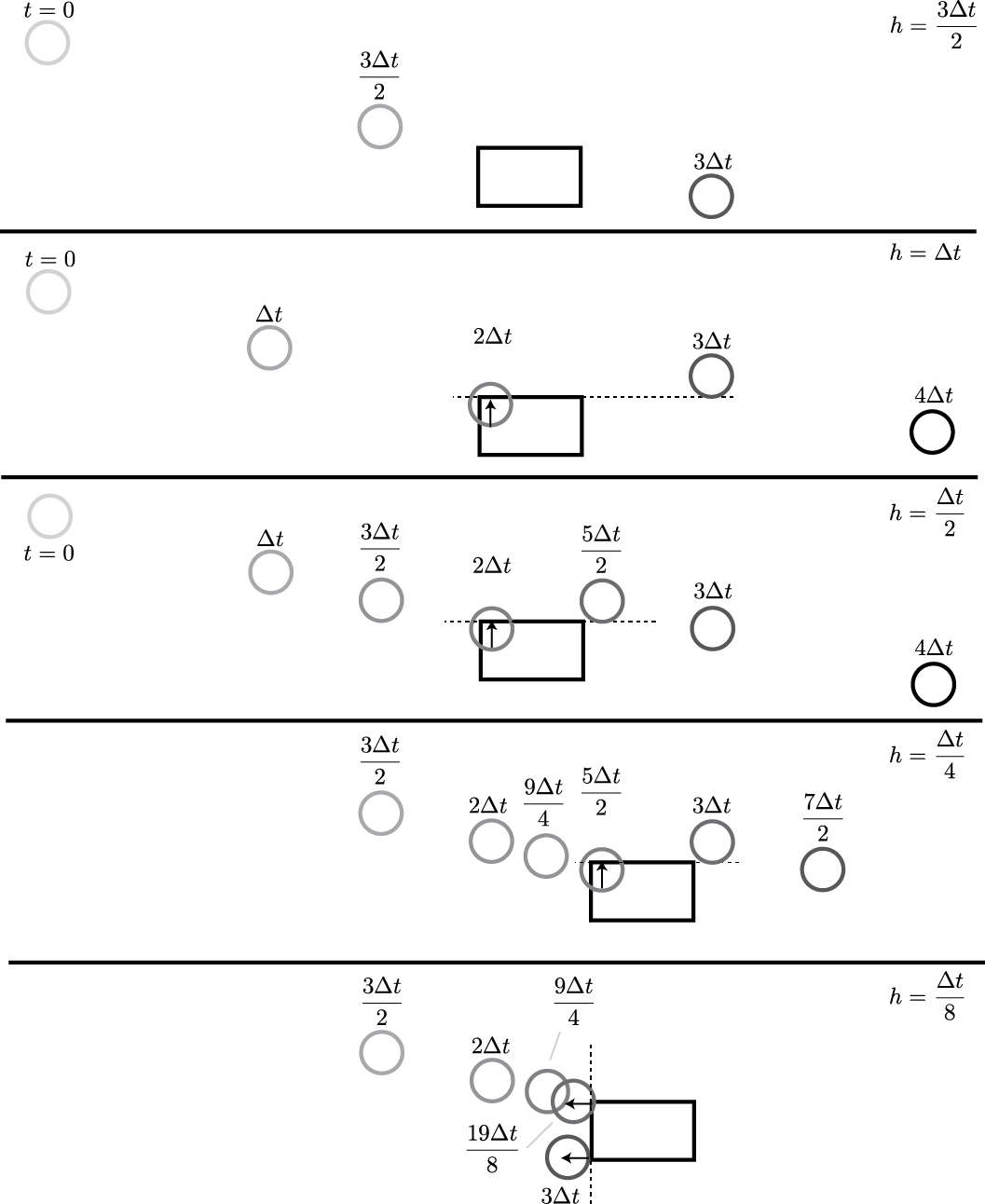}
\caption{Rotated depiction of Anitescu-Potra for simulating a ball moving ballistically toward a stationary box, with different fixed step sizes (top to bottom).  Only some  ball positions are depicted. Contact points lie at the source of arrows; arrows are drawn each time a  contact constrains motion; feasible half spaces are drawn as a line segment with dashed stroke. Contact constraints are determined as described in~\cite{Anitescu:2004} and as implemented in \software{ODE}: contact constraints added at time $t$ only when bodies' geometries are intersecting at $t$. Interpenetration is corrected at time $t+h$, assuming immediate constraint stabilization and that contact constraints are identical at times $t$ and $t+h$. The motion does not qualitatively match the adaptive result (see Figure~\ref{fig:adaptive-approach}) until the step size is $\Delta t/8$.\vspace{-.125in} 
}
\label{fig:AP-approach}
\end{figure}

Proofs of correctness of conservative advancement (i.e., proof that it generates a lower bound on the time of impact) are provided in~\cite{Mirtich:1996vt}. The formula indicates that angular velocities parallel to $\hat{\vect{d}}$---the normal to the contact manifold for ``kissing'' bodies---do not decrease $\Delta t^{\textrm{max}}$, the conservative bound. When bodies are contacting, the safe bound is zero, so a special strategy is necessary. Mirtich kept bodies undergoing sustained contact separated, using a ``microcollision'' approach, which not only results in poor computation rate on some scenarios (like stacking), but also introduces error that does not disappear as the integration step tends to zero. The approach described in the present work instead leverages knowledge that the geometries are convex polyhedra to address this problem. Note that non-convex polytope geometries may be decomposed into unions of convex polytopes (see, e.g.,~\cite{Lien:2008}).

\begin{comment}
\subsection{Quasi-rigid bodies}
Chakraborty et al. have previously investigated quasi-rigid bodies that consist of a rigid core surrounded by a compliant layer~\cite{Chakraborty:2007a}. Slight modeling differences exist between the approach of~\cite{Chakraborty:2007a} and the present work (e.g.,~\cite{Chakraborty:2007a} uses a linear spring-damper model for computing frictional forces); both models suit the purpose of the present work, i.e., inducing distance between the bodies, thereby leading to larger conservative advancement steps. 
\end{comment}

\subsubsection{Convergence rate of conservative advancement}
If bodies initially separated by distance $d$ are bound by conservative advancement as moving toward one another at speed $v$ but are really moving toward one another at true approach speed of $\alpha v$,  for $0 < \alpha < 1$, then the following sequence of integration steps will be produced.
\begin{equation}
d/v, ((d - (d/v)\cdot \alpha v)/v), (d - ((d/v) \cdot \alpha v + ((d - (d/v)\cdot \alpha v)/v) \cdot \alpha v))/v, ...
\end{equation}
This series is geometric. The decreasing distance is also a geometric series: it is the series obtained by multiplying each of the terms above by $v$. For rigid bodies, this series' convergence rate is dependent upon the proportion of linear velocity directed toward the contact surface and the magnitude of angular velocity orthogonal to the contact surface. Distances rapidly converge toward zero for bodies coming into contact, and integration steps increase rapidly for bodies moving apart due to the geometric nature of the series. As Figure~\ref{fig:hist} shows, even fast moving multi-bodies tend to yield large steps.

\begin{SCfigure}
\centering
\includegraphics[width=.7\linewidth]{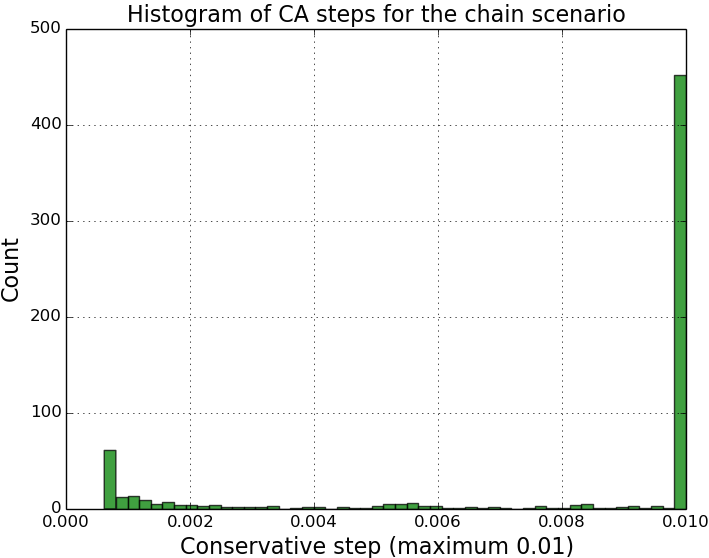}
\caption{The histogram of integration steps for the falling chain scenario (see Section~\ref{section:experiment:chain}) indicates that the maximum integration step is selected the majority of the time.}
\label{fig:hist}
\end{SCfigure}

\section{Computing the contact event times}
Conservative advancement permits finding the first time that two bodies come into contact. Examination of Equation~\ref{eqn:CA} shows that when bodies are touching ($||\vect{d}||=0$), the conservative advancement formula  is unable to provide a next ``safe'' step. This section addresses this problem of finding a next safe step, thereby permitting locating the next time that the contact manifold changes. Computing the times that the contact manifold (seen in Figure~\ref{fig:contact-manifold}) changes relies upon Minkowski's Hyperplane Separation Theorem.

\subsubsection{Hyperplane Separation Theorem}
Let $A$ and $B$ be two disjoint nonempty convex subsets of $\mathbb{R}^n$. Then there exists a non-zero vector $\hat{\vect{n}}$ and a real number $\sigma$ such that:
\begin{align}
\tr{\vect{x}}\hat{\vect{n}} \le \sigma \le \tr{\vect{y}}\hat{\vect{n}}
\end{align}
for all $\vect{x}$ in $A$ and $\vect{y}$ in $B$. The hyperplane normal to $\hat{\vect{n}}$ and defined using $\sigma$ separates $A$ and $B$~\cite{Boyd:2004}. \software{V-Clip} can be used, albeit indirectly, to compute a separating plane in 3D in time $O(m+n)$ in the features of the two polyhedra. 

\begin{SCfigure}
\centering
\includegraphics[width=.4\linewidth]{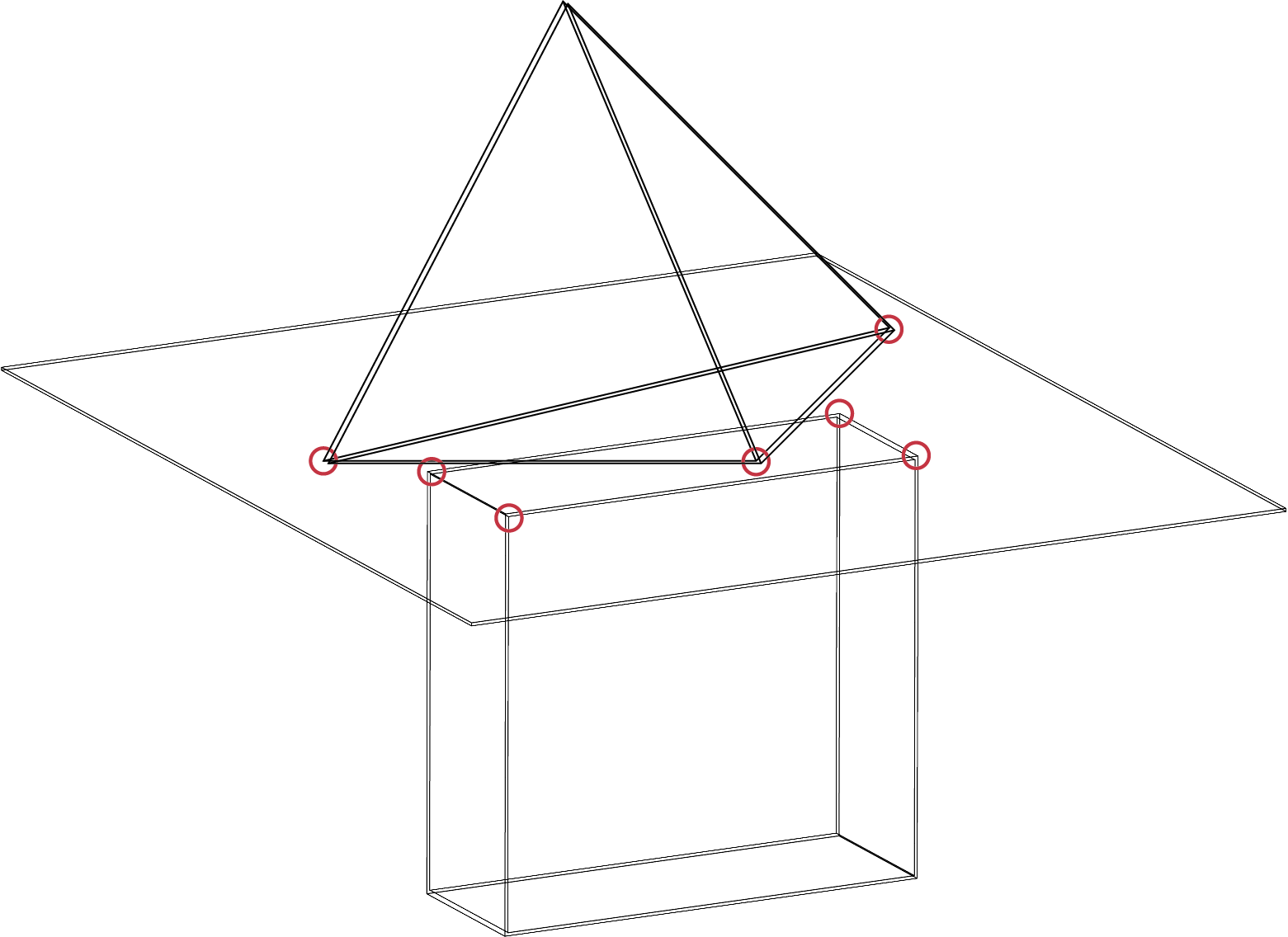}
\caption{The separating plane between two convex polyhedra, a pyramid and a box. Points of contact between the polyhedra and the separating plane are highlighted with red circles.}
\label{fig:contact-manifold}
\end{SCfigure}

\subsubsection{Unique separating hyperplanes}
The separating plane is uniquely defined when two convex polyhedra touch in one of the following ways: at a vertex and the interior of a face, on an edge and a face (with the intersection yielding a line segment), or at two coplanar faces. When two polyhedra touch at two vertices, or at a vertex and the interior of an edge, or at two edges (with the intersection yielding a line segment), the separating plane is not uniquely defined. The latter three cases, which are fortunately rare, are not treated in the present work to avoid overly constraining the motion; these cases can at least be detected and reported, however. This issue is discussed at depth in~\cite{Williams:2014}.

\subsection{Bounding the time that features can newly contact the separating plane}

Assuming that all points touching the separating plane remain on the separating plane, the earliest time that a given vertex $\vect{r}$ of a polyhedron $A$ \emph{not} already touching the separating plane can contact the separating plane can be bound by dividing the projected distance from the vertex to the plane by the motion bound for the two bodies:
\begin{align}
\Delta t^{\textrm{max}^{\vect{r}}} \equiv \frac{|\tr{\hat{\vect{n}}}\vect{r} - \sigma|}{(||\vect{\omega}_A \times \hat{\vect{n}}|| + ||\vect{\omega}_B \times \hat{\vect{n}}||) ||\vect{r} - \vect{\xi}||} \label{eqn:constrained-CA}
%\begin{cases} \frac{|\tr{\hat{\vect{n}}}\vect{r} - \sigma|}{\delta} & \textrm{ if } \delta > 0 \\ \infty & \textrm{ if } \delta \le 0 \end{cases}
\end{align}  
where $\vect{\xi}$ is defined as the point furthest from $\vect{r}$ on the intersection of $A$ and the separating plane   (see Figure~\ref{fig:next-contact}). Ignoring vertices on the contact manifold also avoids undefined values where the equation could evaluate to $0/0$. The bound for when a new vertex, either from $A$ or from $B$, can come into contact with the separating plane is given by:
\begin{align}
\Delta t^* = \min_{\vect{s}} \Delta t^{\textrm{max}^{\vect{s}}},\ \forall \vect{s} \in \textrm{vertices}(A) \cup \textrm{vertices}(B) \textrm{ s. t. } |\tr{\hat{\vect{n}}}\vect{r} - \sigma| > 0
\end{align}
Practically, one only need test the vertices adjacent to the vertices touching the separating plane---by convexity, these adjacent vertices must be at least as close as all other vertices not on the separating plane---but this fact does not alter the worst-case time complexity of this operation (linear in the number of features of the two polyhedra).

\begin{SCfigure}
\centering
\includegraphics[width=.5\linewidth]{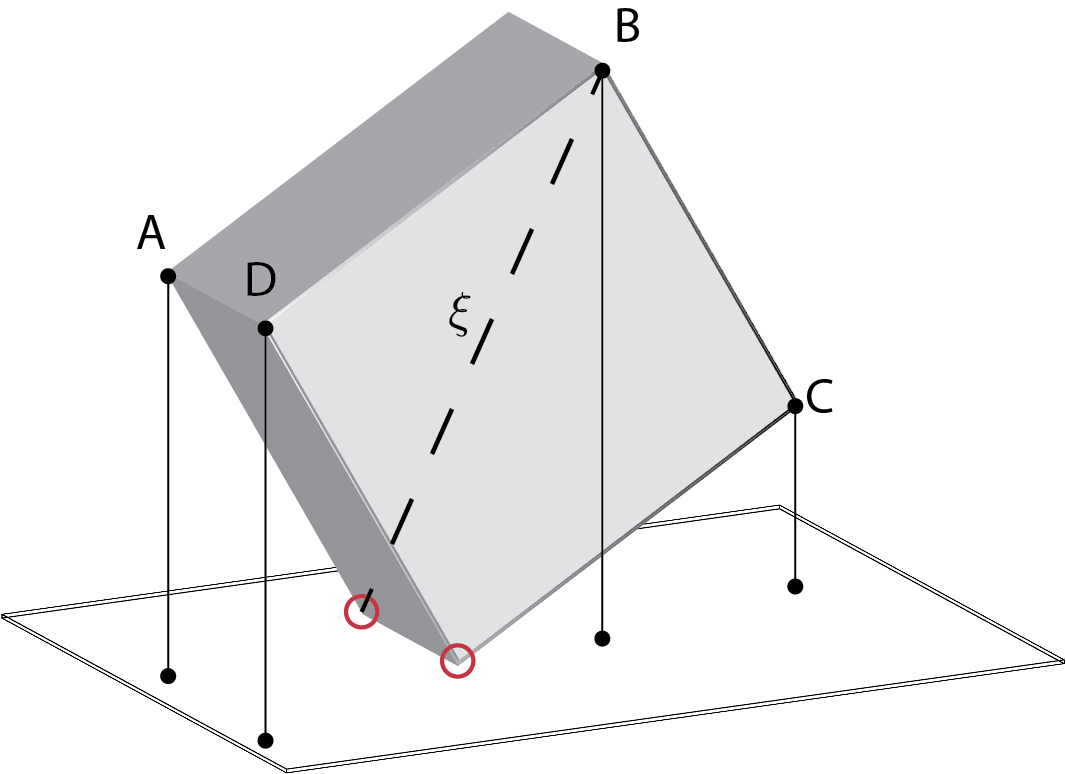}
\caption{Depiction of points from a polyhedron projected onto the separating plane. Dividing the minimum distance from these points to the separating plane by the relative motion that the two bodies can make in unit time yields a safe integration step if the projected \emph{relative} velocities at all contact points (circled in red) are zero at the time pictured.}
\label{fig:next-contact}
\end{SCfigure}

Finally, if should be clear that if the relative velocities between the two polyhedra measured at three non-collinear points on the contact manifold and projected onto the separating plane remain zero throughout [$a,b]$ then any integration step is safe. However, we must conduct the operation described previously as part of the process necessary to ensure that these projected velocities do indeed remain zero. 

\begin{theorem}
If over interval $[a, b]$, no vertex from either polyhedron passes through the separating plane defined at $a$ and moving according to the relative velocities of the two polyhedra over $[a, b]$, the two bodies cannot interpenetrate in $[a,b]$.
\end{theorem}

\begin{proof}
Follows directly from the Separating Axis Theorem.
\end{proof}

\begin{theorem}
A feature from polyhedron $A$ that lies strictly above/below the separating plane at time $t_0$ must lie strictly above/below the separating plane during the half-closed, ``safe'' interval $[t_0, t_0 + \Delta t^*)$.
\end{theorem}

\begin{proof}
The numerator of $\Delta t^*$ is comprised from the minimum distance between any feature on the body and the separating plane: if an oracle were to produce the ``true'' value of the numerator ($e$, hereafter), that value would be at least as large.  On the other hand, the denominator ($f$, hereafter)---which bounds the motion of the feature toward the separating plane---provides an upper limit. For any $0 \le \delta^e$, $0 \le \delta^f < f$ yielding the time of contact $t_0 + \frac{e+\delta^e}{f-\delta_f}$, it is clear that $\Delta t^* \le \frac{e+\delta^e}{f-\delta_f}$.
\end{proof}

\subsection{Providing a lower bound on the time vertices remain on the separating plane}
$\Delta t^*$ bounds when vertices not already on the separating plane might contact it. For those vertices already on the separating plane, we must examine the velocities projected onto the separating plane at $t_0$ to obtain $\Delta t^\dag$, a lower bound on the time that vertices remain on the (possibly moving) separating plane. $\dot{\phi}^{\vect{p}}$ will hereafter be denoted the relative velocity of the two bodies at point $\vect{p}$ projected onto the separating plane at $t_0$.

\begin{figure}[htbp]
\centering
\includegraphics[width=\linewidth]{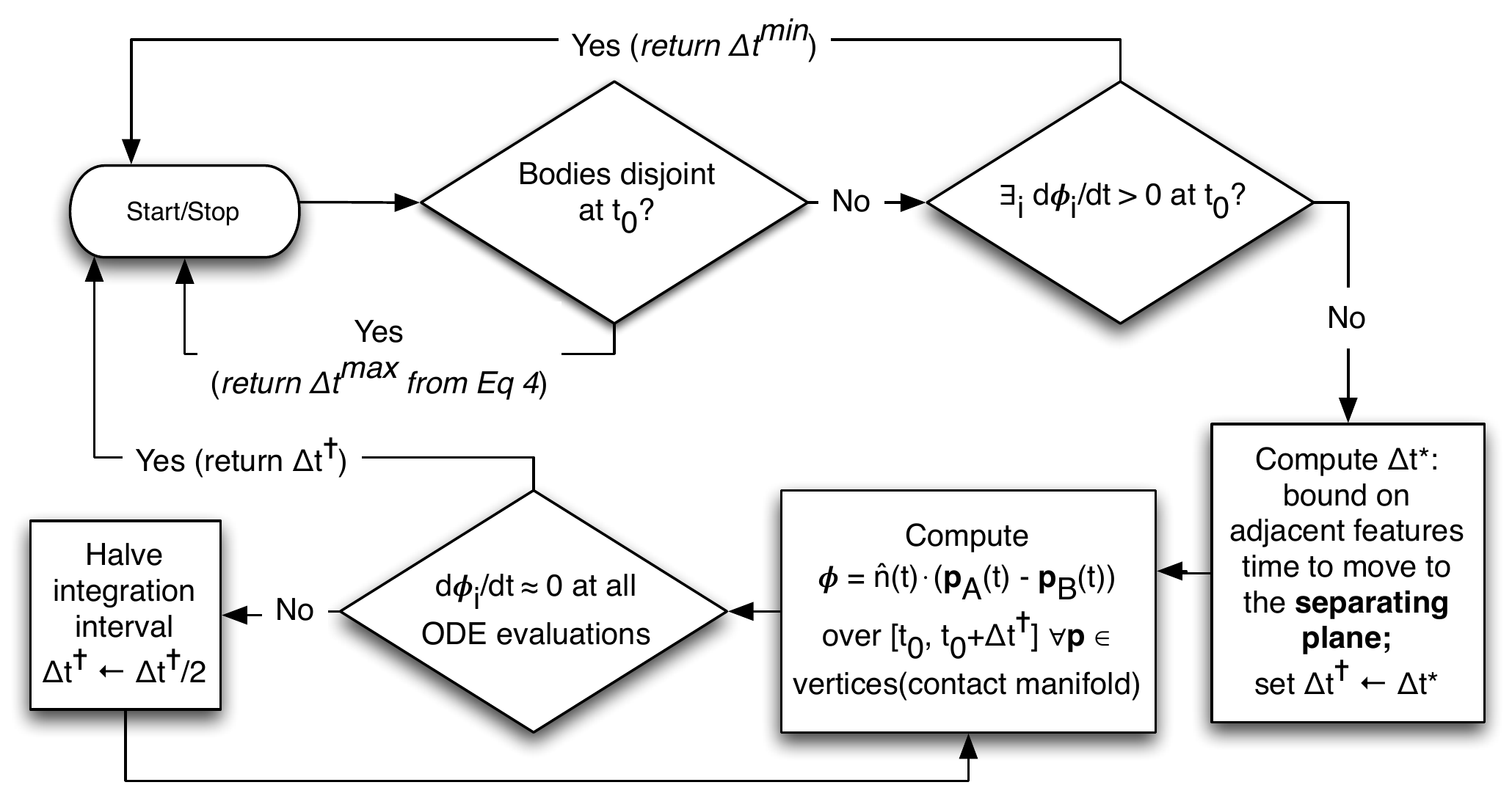}
\caption{Flowchart for determining $\Delta t = r(\vect{q}, \vect{v})$ (see Algorithm~\ref{alg:ts}).}
\label{fig:process}
\end{figure}

The formal definition of $\dot{\phi}^{\vect{p}}$ arises from the distance between a shared point on Body $A$ and Body $B$, projected along the normal to the separating plane $\hat{\vect{n}}$:
\begin{align}
\phi^{\vect{p}} \equiv \tr{\hat{\vect{n}}} (\vect{p}_A - \vect{p}_B)
\end{align}
Constraint solvers (e.g., LCP codes) can ensure that the time derivative of this equation is non-negative at time $t_0$ (denoted henceforth as the time of contact)\footnote{Typical time stepping equations can be used for this purpose (see, e.g.,~\cite{Anitescu:1997}) as long as zero is substituted for $h$.}:
\begin{align}
\dot{\phi^{\vect{p}}} = \tr{\hat{\vect{n}}} (\dot{\vect{p}}_A - \dot{\vect{p}}_B) + \tr{\dot{\hat{\vect{n}}}} (\vect{p}_A - \vect{p}_B)
\end{align}
When the second term above is evaluated at $t_0$, point $\vect{p}$ is shared between the two polyhedra, so $\vect{p}_A - \vect{p}_B = \vect{0}$, i.e., the second term goes to zero.

We now consider two cases: \1 the instantaneous velocities of all vertices projected onto the separating plane are zero and \2 the instantaneous velocity of one or more vertices projected onto the separating plane is non-zero. In the latter case, the multi-bodies must be integrated by some small step, $\Delta t^{\textrm{min}}$, until those active constraints (for the vertices with non-zero project velocity) become inactive as the distances between those vertices and the separating plane become non-zero (refer to Equation~\ref{eqn:DAE1}).  In the former case, we seek the first time (if any) that the time derivative of the signed distance of a vertex from the separating plane becomes non-zero (i.e., $\dot{\phi^{\vect{p}}}(t) \neq 0$).

$\phi^{\vect{p}}$ can be integrated over an interval \emph{free of changes to the active constraint set} using the ODE above (which requires integrating the equations of motion for Bodies $A$ and $B$). Given that the active constraint set remains constant over the interval $[a, b]$, variable step ODE integrators can compute $\phi^{\vect{p}}$ over $[a, b]$ such that desired error tolerances are met. \emph{If all evaluations of $\dot{\phi}^{\vect{p}}$ that are used to compute $\phi^{\vect{p}}$ are approximately zero, then $\phi^{\vect{p}}(t) \approx 0,  \forall t \in [a, b]$}, subject to the specified error tolerances. If not all evaluations are zero, the active constraint set is going to change, so the step size is halved until all evaluations are zero. If all evaluations are zero but the error tolerances are not satisfied, the error estimates are used to select a new integration step (as described in~\cite{Hairer:2008}, Ch. II.4) to meet the desired tolerances. $\Delta t^\dag$ is determined as depicted in Figure~\ref{fig:process}. 

\begin{theorem}
The contact manifold cannot change in the half-open interval\\ \mbox{$[t_0, t_0 + \min{(\Delta t^*, \Delta t^\dag)})$}.
\end{theorem}

\begin{proof}
This proof relies upon the observation that the contact manifold does not change from $t_0$ until a vertex from a polyhedron leaves or touches the separating plane. $\Delta t^*$ bounds the time that no vertex from either polyhedron can newly contact the separating plane, and $\Delta t^\dag$ bounds the time that no vertex can leave the separating plane. 
\end{proof}

\section{Integration approach}

The approach for integrating bodies over the time interval $[a, b]$ uses Mirtich's CA approach to integrate bodies' states to some time $a < t_0 \le b$ such that the minimum distance between the bodies at time $t_0$ is $\epsilon \ll 1$ (with $\epsilon$ possibly equal to zero). The process described in the previous section takes over at that point, and ensures that ensuing steps are sufficiently small to detect changes to the contact manifold.

\subsection{Conservative advancement and integration}
Algorithm~\ref{alg:ts} describes a first-order integrator. CA can function with higher order integration schemes~\cite{Mirtich:1996vt}, though computational efficiency is reduced (the integrator may have to re-integrate the same time interval repeatedly). 

\begin{algorithm}[htpb]
\caption{\label{alg:ts} \textsc{Simulate}($t_i, \Delta t$) Simulates a system of multi-rigid bodies from time $t_i$ to time $t_f \le t_i + \Delta t$ using a step size such that contact events are not missed, to the accuracy of the first order integrator. This algorithm assumes that all bodies are at admissible configurations at $\{ t_i, \vect{q}_i, \vect{v}_i \}$; if the initial state of the simulator is feasible, \textsc{Simulate} will keep the constraints feasible thereafter. }
\begin{algorithmic}[1]
\State $h \leftarrow \min{(r(\vect{q}_{i}, \vect{v}_i), \Delta t)}$
\Comment Compute conservative step (see Figure~\ref{fig:process})
\State $\vect{q}_{i+1} \leftarrow \vect{q}_i + h\vect{v}_i$ \Comment Integrate position forward
\State $\vect{v}^*_{i+1} \leftarrow \vect{v}^*_i + h\vect{f}(\vect{q}_{i+1},\vect{v}_i)$ \Comment Integrate velocity forward 
\State $\vect{v}_{i+1} \leftarrow \vect{k}(\vect{v}_{i+1}^*)$ \Comment Determine contacts (if any) and apply constraint impulses
%\State $\vect{q}_{i+1} \leftarrow \vect{y}(\vect{q}_{i+1})$ \Comment Apply projective constraint stabilization to bodies' coordinates
\State \textbf{return} $\{ t_i + h, \vect{q}_{i+1}, \vect{v}_{i+1} \}$
\end{algorithmic}
\end{algorithm}

By using Equation~\ref{eqn:CA} to select the integration step, denoted as $h$, bodies will not be interpenetrating (excepting small, floating point rounding error) at the end of the step. The function $\vect{f}(.)$ on Line~3 is defined as the ordinary differential equation for the equations of motion, which may be defined in minimal or absolute coordinates; the constraint equations are not considered at this point.  Rather, bilateral and unilateral constraint forces are computed on Line~4. The constraint model from~\cite{Drumwright:2015} is applied in the experiments in Section~\ref{section:experiments}, though the models from~\cite{Stewart:1996,Anitescu:1997} are usable as well. 

\begin{theorem}
A positive integration step is always possible.
\end{theorem}

\begin{proof}
If bodies' geometries are disjoint at $t_0$, Equation~\ref{eqn:CA} must return a positive value (infinity if, e.g., the bodies are motionless). If bodies are in sustained contact, Equation~\ref{eqn:constrained-CA} will yield a positive value: the numerator must be positive---non-positive values correspond to vertices that lie on the contact manifold and are not evaluated by this equation---and the denominator must evaluate to a non-negative value. If contact between bodies is transient, a small positive value is returned (see Figure~\ref{fig:process}).
\end{proof}

\begin{figure}[htb]
\centering
\includegraphics[width=.8\linewidth]{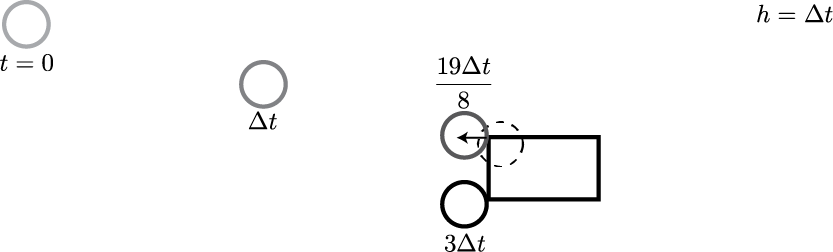}
\caption{The adaptive approach described in this paper using a maximum step size of $h = \Delta t$, applied to the scenario depicted in Figures~\ref{fig:ST-approach} and~\ref{fig:AP-approach}. Adaptive stepping avoids the problems exhibited for time stepping with $h = 3\Delta t/2$ (missing the contact completely) and $h=\{ \Delta t, \Delta t/2, \Delta t/4 \}$ (contact is constrained in the wrong direction); see Figure~\ref{fig:ST-approach}. }
\label{fig:adaptive-approach}
\end{figure}

\subsubsection{Determining contact data}
\label{section:research:determining-contact-data}
When bodies are separated by distance $\epsilon$ or less, contact points and surface normals are determined by examining closest features. For polyhedra ``kissing'' (to floating point tolerances), these features are vertices, edges, and faces, and a pair of closest features might form a contact point, edge, or surface. The intersection of two convex polytopes results in a convex polytope of $O(n+m)$ vertices. Thus, solving LCPs corresponding to contact models on these inputs will exhibit running time $O(n^3)$, assuming $O(n)$ pivots for Lemke's Algorithm~\cite{Cottle:1992}. Convexity implies that velocities that satisfy the non-interpenetration constraints at these $O(n+m)$ vertices also satisfy non-interpenetration constraints anywhere on the convex contact surface. Thus as long as the constraint algorithm is guaranteed to find a solution for forces at the $O(n+m)$ contact points---approaches in~\cite{Stewart:1996,Anitescu:1997,Drumwright:2010b}, for example, provide such guarantees---two polyhedra will remain disjoint through  some time $t > t_0$. Space limitations prevent further discussion of contact determination.

\section{Conservative advancement for multibodies}
\label{section:multibody-CA}
Determining the distance a link on a multibody can move in unit time along a vector $\hat{\vect{d}}$ requires special consideration. %solving a problem similar to the \emph{reachability problem}~\cite{Kumar:1981} in robotics; determining a link's possible movement along a direction requires considering the geometry of the 
If Body $A$ is the $i^{\textrm{th}}$ link in an articulated body chain (and attached to its parent by joint $i$), the terms that include Body $A$ in the denominator of Equation~\ref{eqn:CA} would change to:
\begin{align}
|\tr{\dot{\vect{x}}}\hat{\vect{d}}| + ||\vect{\omega}|| \gamma_1 + \sum_{j=1}^i \kappa_j\label{eqn:CA-multibody}
\end{align}
where
% gamma is the length from the joint to the c.o.m. plus radius
% kappa is the contribution of joint i
\begin{align}
\kappa_j & = \begin{cases} 
|\dot{q}_j|& \textrm{ if joint $j$ prismatic} \\
||\dot{q}_j|| \gamma_j & \textrm{ if joint $j$ revolute, spherical, or universal}
\end{cases} \\
\gamma_j & \equiv \sum_{k=j}^{i-1} d^\ell_k + 2r^A \label{eqn:gamma} \\
d^\ell_k &= \begin{cases} 
|q_k| + |\dot{q}_k|  & \textrm{ if joint $k$ prismatic} \\
||\vect{\ell}_{k+1} - \vect{\ell}_{k}|| & \textrm{ if joint $k$ revolute, spherical, or universal}
\end{cases}
\label{eqn:d}
\end{align}
where $\dot{\vect{x}}$ and $\vect{\omega}$ are the linear and angular velocities of the robot's base (which will be zero if the base is affixed to the environment) and $\vect{\ell}_j$ is the location of the $j^{\textrm{th}}$ joint (in the global frame) when all prismatic joints are configured to yield the minimum Euclidean distance between successive joints. We determined these formulas by analyzing the illustration in Figure~\ref{fig:CA} (R); they could also be derived by considering the possible Jacobian matrices with respect to \emph{any point} on a link with spherical geometry; space limitations preclude a full derivation. We have tested that the formulas provide correct bounds on multi-bodies with both revolute and prismatic joints.

Similarly, the equation below replaces Equation~\ref{eqn:constrained-CA} (i.e., the constrained conservative advancement formula) when Body A is a link in a multibody:
\begin{align}
\Delta t^{\textrm{max}^{\vect{r}}} \equiv \frac{|\tr{\hat{\vect{n}}}\vect{r} - \sigma|}{(||\vect{\omega}_A|| + \sum_{j=1}^i ||\dot{\vect{q}}_j|| + ||\vect{\omega}_B \times \hat{\vect{n}}||) ||\vect{r} - \vect{\xi}||}
%\begin{cases} \frac{|\tr{\hat{\vect{n}}}\vect{r} - \sigma|}{\delta} & \textrm{ if } \delta > 0 \\ \infty & \textrm{ if } \delta \le 0 \end{cases}
\end{align}

\begin{figure}[htbp]
\centering
\includegraphics[width=.295\linewidth]{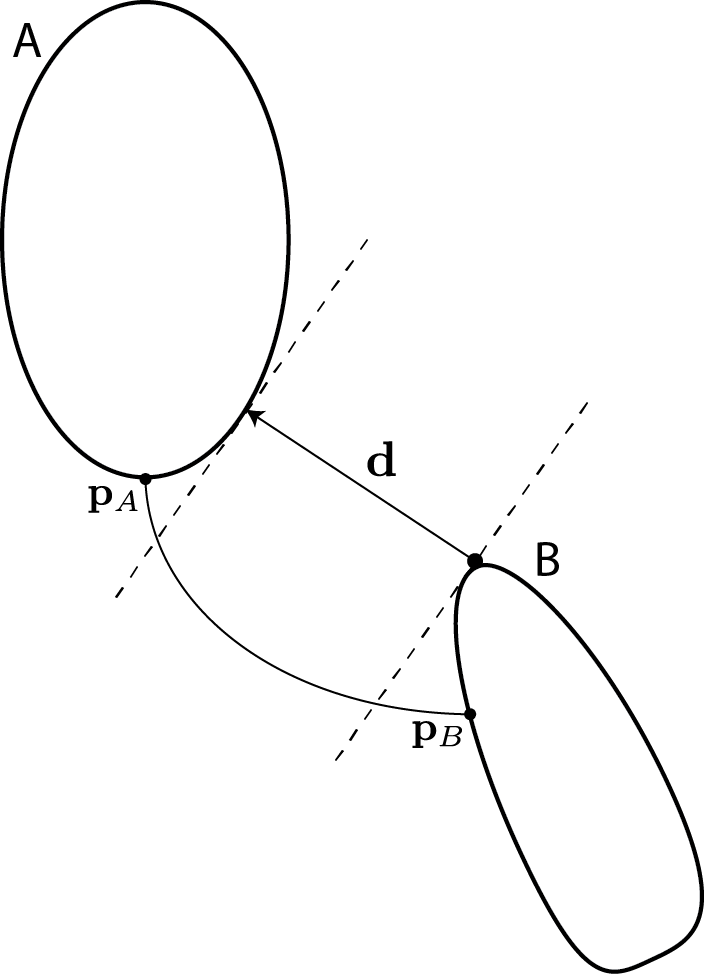} \hspace{.25in}
\includegraphics[width=.6\linewidth]{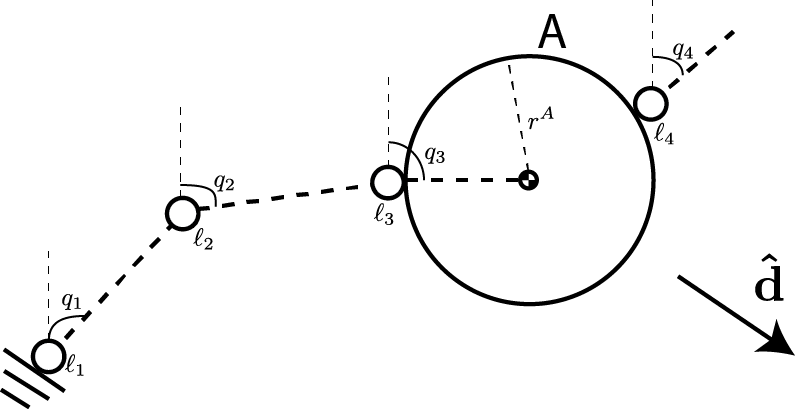}\caption{(Left) Conservative advancement (CA) uses the distance between two rigid bodies to bound the motion between bodies over a time interval. For points $\vect{p}_A$ and $\vect{p}_B$ to collide, they must cover at least distance $||\vect{d}||$ along direction $\vect{d}$.   (Right) Our extension of CA to a multi-body uses Equations~\ref{eqn:CA-multibody}--\ref{eqn:d} to bound the movement of Body A (Link $i$) over a unit time interval.
\label{fig:CA}}
\end{figure}

\section{Experiments}
\label{section:experiments}
We have implemented this research in the open source library, \software{Moby}.  We use a first-order, velocity-level approach with a no-slip contact model (described in~\cite{Drumwright:2015}) to obviate the issue of indeterminate contact configurations. Constraint stabilization---which is generally used by simulators to correct interpenetration, albeit not without introducing new problems~\cite{Ascher:1995a,Drumwright:2016}---is disabled in these experiments to gauge constraint violation in its absence. Experimental code is available for review at \url{https://github.com/PositronicsLab/wafr16-experiments}.

\subsection{Rotating box experiment}
We constructed the scenario of a die thrown onto a planar surface to 
assess the efficiency of detecting changes to the contact manifold. The initial position, orientation, and velocity of the die was set to an arbitrary (randomly selected) value, and the simulation was run until the die became motionless. The simulation detected 102 changes to the contact manifold for this scenario; examination indicates that the large number of changes was due to chattering. The minimum distance observed between the die and the half-space was $1.0 \times 10^{-8}$~m (i.e., there was no interpenetration). The mean time between contact manifold changes was $5.3 \times 10^{-3}$~s, while the mean integration step was $2.7 \times 10^{-5}$~s. The minimum time between contact manifold changes was $1.0 \times 10^{-6}$~s. Therefore, the adaptive simulation technique took, on average, a step 27 times larger than the minimum necessary to detect all contact manifold changes and a mean step 196 times smaller than the largest mean step for detecting contact manifold changes. 

\subsection{Chain experiment}
\label{section:experiment:chain}
Another experiment was run for which it was known that there would be zero contact manifold changes because contact never occurs. The experiment uses a $100$ m chain suspended $100.0001$ m above the ``ground'', extended $100$ m horizontally, and then dropped; the chain is essentially a 100-link pendulum. The simulation was run until immediately after the chain skims the ground. When the chain reaches this nadir, the tip is moving at high speed, stressing the multi-body conservative advancement described in Section~\ref{section:multibody-CA}. The smallest integration step recorded is $6.17 \times 10^{-4}$ s, and the mean integration step is $0.0065$ s. The nominal integration step is $0.01$: the adaptive process uses a step  $\frac{1}{2}$ times smaller than the optimum step for detecting contact manifold changes. Figure~\ref{fig:hist} shows a histogram of integration step sizes.

\subsection{Quadrupedal trotting experiment}
We used the scenario of a quadrupedal robot trotting on a halfspace to assess the efficiency of detecting changes to the contact manifold for a multi-body undergoing sustained contact. The 18 DoF robot model has been used in prior work (see, e.g., \cite{Zapolsky:2014}), and uses spherical polyhedron feet with 1,000 vertices each. Gait generation and control is provided by \software{Pacer}~\cite{Zapolsky:2015a}. The robot was simulated trotting for a single gait cycle. The simulation detected 317 changes to the contact manifold for this scenario. The mean time between contact manifold changes was $1.5 \times 10^{-3}$~s, while the mean integration step was $7.2 \times 10^{-6}$~s. The minimum time between contact manifold changes was $1.3 \times 10^{-7}$~s. Therefore, the adaptive simulation technique took, on average, a step size $55$ times larger than the minimum necessary to detect all contact manifold changes and a mean step $213$ times smaller than the largest mean step for detecting contact manifold changes. 

\section{Discussion}

This work aims to provide the correct result, to a desired tolerance, to the mathematical equations for rigid and multi-rigid bodies with convex polytopic geometries. Whether such rigidity is physically accurate or computationally efficient are inapposite, albeit interesting, questions. 

Two pieces of work remain to achieve this goal for general multi-body simulations. First, all events that correspond to \1 hitting joint limits and to \2 transitioning from sticking to sliding (and vice versa) for dry contact friction must be located. Locating both types of events superficially appears much easier than locating contact events. The second task remaining is that of incorporating the integration error estimates, produced through adaptive integration, into the conservative advancement process. 

Finally, a caveat remains for users that expect a solution to always be  obtainable. Indeterminacy in the rigid contact model with Coulomb friction means that a single solution (i.e., a single trajectory) may not exist; instead, the solutions to some problems may be characterized by ``trees'', some of which may possess an infinite branching factor. 

\section*{Acknowledgements}
The author thanks Russ Tedrake and Michael Sherman for excellent feedback. This work was funded by ARO grant W911NF-16-1-0118.

\bibliographystyle{abbrv}
\bibliography{paper}

\end{document}